\documentclass[runningheads]{llncs}
\usepackage{graphicx}
\usepackage{subfigure}
\usepackage{comment}
\usepackage{amsmath,amssymb} 
\usepackage{color}
\usepackage{bm}
\usepackage{commath}
\usepackage{verbatim}
\usepackage{algorithm,algorithmic}
\usepackage{multicol}
\newtheorem{assumption}{Assumption}
\newcommand{\xbf}{\mathbf{x}}
\newcommand{\ybf}{\mathbf{y}}
\newcommand{\zbf}{\mathbf{z}}
\newcommand{\bbf}{\mathbf{b}}
\newcommand{\hbf}{\mathbf{h}}

\newcommand{\ubf}{\mathbf{u}}
\newcommand{\vbf}{\mathbf{v}}
\newcommand{\reta}{r_{\eta}}

\DeclareMathOperator{\prox}{\mathrm{prox}}
\DeclareMathOperator*{\argmin}{\mathrm{arg\,min}}
\DeclareMathOperator*{\argmax}{\mathrm{arg\,max}}

\newcommand{\Phibf}{\mathbf{\Phi}}

\begin{document}
\pagestyle{headings}
\mainmatter
\def\ECCVSubNumber{370}  
\title{A Novel Learnable Gradient Descent Type Algorithm for Non-convex Non-smooth Inverse Problems} 

\titlerunning{ResGD-Net}
%
\author{Qingchao Zhang\inst{1}\and
Xiaojing Ye\inst{2} \and
Hongcheng Liu\inst{1} \and
Yunmei Chen \inst{1}}
\authorrunning{Q. Zhang et al.}
%
\institute{University of Florida, Gainesville, Florida 32611, USA \and
Georgia State University, Atlanta, Georgia 30303, USA
\\
\email{ $^1$\{qingchaozhang,liu.h,yun\}@ufl.edu, $^2$xye@gsu.edu} }
\maketitle

\begin{abstract}
Optimization algorithms for solving nonconvex inverse problem have attracted significant interests recently. However, existing methods require the nonconvex regularization to be smooth or simple to ensure convergence.
In this paper, we propose a novel gradient descent type algorithm, by leveraging the idea of residual learning and Nesterov's smoothing technique, to solve inverse problems consisting of general nonconvex and nonsmooth regularization with provable convergence.
Moreover, we develop a neural network architecture intimating this algorithm to learn the nonlinear sparsity transformation adaptively from training data, which also inherits the convergence to accommondate the general nonconvex structure of this learned transformation.
Numerical results demonstrate that the proposed network outperforms the state-of-the-art methods on a variety of different image reconstruction problems in terms of efficiency and accuracy.

\end{abstract}

\section{Introduction}
\label{sec:intro}
These years have witnessed the tremendous success of deep learning in a large variety of real-world application fields \cite{DLH14,HSW89,LPW17,Yarotsky}.
At the heart of deep learning are the deep neural networks (DNNs) which have provable approximation power and the substantial amount of data available nowadays for training these DNNs.
%
Deep learning can be considered as a data-driven approach since the DNNs are mostly trained with little or no prior information on the underlying functions to be approximated.
However, there are several major issues of generic DNNs that have hindered the application of deep learning in many scientific fields: (i) Generic DNNs may fail to approximate the desired functions if the training data is scarce; (ii) The training of these DNNs are prone to overfitting, noises, and outliers; (iii) The result DNNs are mostly ``blackboxes'' without rigorous mathematical justification and can be very difficult to interpret.

Recently, learned optimization algorithm (LOA) as a promising approach to address the aforementioned issues has received increasing attention.
LOA is aimed at combining the best of the mathematically interpretable optimization algorithms and the powerful approximation ability of DNNs, such that the desired functions can be learned by leveraging available data effectively.
In particular, an LOA is often constructed by unrolling an iterative optimization algorithm, such that one or multiple layers of the LOA correspond to one iteration of the algorithm, and the parameters of these layers are then learned from data through the training process.

In the field of computer vision and image processing, most existing optimization algorithms are developed based on either smooth or convex objective functions with relatively simple, handcrafted structures.
The schemes and convergence of these algorithms heavily rely on the strict assumptions on these structures.
However, the networks in the corresponding LOAs are trained to have rather complex, nonsmooth \emph{and} nonconvex structures.
In this case, the LOAs only have superficial connections to the original optimization algorithms, and there are no convergence guarantee on these LOAs due to the learned complex structures.

The goal of this paper is to develop a gradient descent type optimization algorithm to solve general nonsmooth and nonconvex problems with provable convergence, and then map this algorithm to a deep reconstruction network, called ResGD-Net, that can be trained to have rather complex structures but still inherit the convergence guarantee of the algorithm. Our method possesses the following features: (i) We tackle the nonsmooth issue of the optimization problem by the Nesterov's smoothing technique \cite{nesterov2005smooth} with rigorous, provable convergence; (ii) We employ an iterate selection policy based on objective function value to safeguard convergence of our method; (iii) We integrate the residual network structure \cite{ResNet} into the proximal gradient scheme of our algorithm for improved efficiency in network training.

The remainder of this paper is organized as follows. In Section \ref{sec:related}, we review the recent literature on learned optimization algorithms. In Section \ref{mainsec}, we present our gradient descent type algorithm for solving general nonconvex and nonsmooth optimization problems, and map it to a deep neural network that allows the regularization term to be learned from training data. The convergence and complexity analysis are also provided. In Section \ref{experiment}, we conduct a number of numerical experiments on natural and medical image reconstruction problems to show the promising performance of our proposed method. We provide several concluding remarks in Section \ref{sec:conclusion}.

\section{Related Work}
\label{sec:related}
The majority of computer vision and imaging problems are formulated as regularized inverse problems as follows:
\begin{equation}\label{eq:cp}
\min_\xbf f(\xbf;\zbf) + r(\xbf),
\end{equation}
where $f$ is the data fidelity term that measures the discrepancy between the candidate solution $\xbf$ and the observed data $\mathbf{z}$, and $r$ is a regularization term that imposes prior knowledge or preference on the solution $\xbf$.
The regularization term $r(\xbf)$ is critical to obtain high quality solution from \eqref{eq:cp}, as the data fidelity $f$ is often underdetermined, and the data $\zbf$ can be incomplete and noisy in real-world applications.
In the inverse problem literature, $r$ is often handcrafted and has simple structure so that the problem \eqref{eq:cp} can be relatively easy to solve with convergence guarantee.
However, these simple handcrafted regularization terms may not be able to capture the complex features of the underlying solution $\xbf$, and hence \eqref{eq:cp} produces undesired results in practice.
This motivates the study of LOAs in recent years which replace the handcrafted components with trained ones by leveraging the large amount of data available.

Existing LOAs can be approximately categorized into two groups.
The first group of LOAs appeared in the literature are motivated by the similarity between the iterative scheme of a traditional optimization algorithm (e.g., proximal gradient algorithm) and a feed forward neural network.
Provided instances of training data, such as ground truth solutions, an LOA replaces certain components of the optimization algorithm with parameters to be learned from the data.
The pioneer work \cite{KY2010} in this group of LOAs is based on the well-known iterative shrinkage thresholding algorithm (ISTA) for solving the LASSO problem $\min_\xbf (1/2)\cdot \|\Phi \xbf-\zbf\|^2+\lambda \|\xbf\|_1$ by iterating $\xbf^{k+1}=\mathrm{shrink}(\xbf^k-\tau \Phi ^{\top}(\Phi \xbf^k-\zbf); \lambda \tau)$, where $\tau \in (0,1/\|\Phi ^{\top} \Phi \|]$ is the step size, and $[\mathrm{shrink}(\xbf;\lambda)]_i=\mathrm{sign}(x_i) \cdot \max(0, |x_i|-\lambda)$ for $i=1,\dots,n$ represents the component-wise soft shrinkage of $\xbf=(x_1,\dots,x_n)$.
In \cite{KY2010}, a learned ISTA network, called LISTA, is proposed to replace $\Phi^{\top
}$ by a weight matrix to be learned from instance data to reduce iteration complexity of the original ISTA.
The asymptotic linear convergence rate for LISTA is established in \cite{CLWY2018} and \cite{LCW2019}.
Several variants of LISTA were also developed using low rank or group sparsity \cite{SBS2015},  $\ell_0$ minimization \cite{XWG16} and learned approximate message passing \cite{BSR2017}.
The idea of LISTA has been extended to solve composite problems with linear constraints, known as the differentiable linearized alternating direction method of multipliers (D-LADMM) \cite{XWZ2019}.
These LOA methods, however, still employ handcrafted regularization and require closed form solution of the proximal operator of the regularization term.

The other group of LOAs follow a different approach to solve the inverse problem \eqref{eq:cp} with regularization term $r$ learned from training data.
The goal of these LOAs is to replace the handcrafted regularization $r$, which is often overly simplified and not able to capture the complex features of the solution $\xbf$ effectively, by employing multilayer perceptrons (MLP) adaptively trained from data.
Recall that a standard approach to solving \eqref{eq:cp} is the proximal gradient (PG) method:
\begin{equation}\label{PG}
\xbf^{k+1} = \prox_{\alpha _k R}(\bbf^k):=\argmin_\xbf\ \frac{1}{2}\|\xbf-\bbf^k\|^2+ \alpha_k r(\xbf),
\end{equation}
where $\bbf^k = \xbf^k - \alpha_k \nabla f(\xbf^k; \zbf)$ and $\alpha_k > 0$ is the step size in the $k$th iteration.
Learning regularization $r$ in \eqref{eq:cp} effectively renders the proximal term $\prox_{\alpha_k r}$ in \eqref{PG} being replaced by an MLP.
Therefore, one avoids explicit formation of the regularization $g$, but creates a neural network with prescribed $K$ phases, where each phase mimics one iteration of the proximal gradient method \eqref{PG} to compute $\bbf_{k}$ as above and $\xbf_{k}=\hbf_{k}(\bbf_{k})$.
The CNN $\hbf_{k}$ can also be cast as a residual network (ResNet) \cite{ResNet} to represent the discrepancy between $\bbf_{k}$ and the improved $\xbf_{k}$ \cite{zhang2017learning}.
Such a paradigm has been embedded into  half quadratic splitting in DnCNN \cite{zhang2017learning}, ADMM in \cite{chang2017one,meinhardt2017learning} and primal dual methods in \cite{AO18,LCW2019,meinhardt2017learning,wang2016proximal} to solve the subproblems. 
To improve over the generic black-box CNNs above, several LOA methods are proposed to unroll numerical optimization algorithms as deep neural networks so as to preserve their efficient structures with proven efficiency, such as the ADMM-Net \cite{NIPS2016_6406} and ISTA-Net \cite{Zhang2018ISTANetIO}.
These methods also prescribe the phase number $K$, and map each iteration of the corresponding numerical algorithm to one phase of the network, and learn specific components of the phases in the network using training data.

Despite of their promising performance in a variety of applications, the second group of LOAs only have superficial connection with the original optimization algorithms.
These LOAs lose the convergence guarantee due to the presence of complex nonconvex and/or nonsmooth structures learned from data.
Moreover, certain acceleration techniques proven to be useful for numerical optimization algorithms are not effective in their LOA counterparts.
For example, the acceleration approach based on momentum \cite{Nesterov} can significantly improve iteration complexity of traditional (proximal) gradient descent methods, but does not have noticeable improvement when deployed in the network versions.
This can be observed by the similar performance of ISTA-Net \cite{Zhang2018ISTANetIO} and FISTA-Net \cite{zhang2017learning}.
One possible reason is that the LOA version has learned nonconvex components, for which a linear combination of $\xbf^k$ and $\xbf^{k-1}$ is potentially a worse extrapolation point in optimiztaion \cite{li2015accelerated}.
On the other hand, several network engineering techniques are shown to be very effective to improve practical performance of LOAs.
For example, ISTA-Net$^+$ \cite{Zhang2018ISTANetIO} employs the residual network structure \cite{ResNet} and results in substantially increased reconstruction accuracy over ISTA-Net.
The residual structure is also shown to improve network performance in a number of recent work, such as ResNet-v2 \cite{he2016identity}, WRN \cite{zagoruyko2016wide}, and ResNeXt \cite{xie2017aggregated}.

\section{A Novel Gradient Descent Type Algorithm } \label{mainsec}
In this section, we present a novel gradient decent type algorithm to solve the general nonsmooth and nonconvex optimization problem with focus application on image reconstruction:
\begin{equation}\label{eq:F}
  \min_{\xbf \in \Re^{n}}\{ F(\xbf) := f(\xbf) + r(\xbf )\},
\end{equation}
where $f$ is the data fidelity term (we omit the notation $\zbf$ as the data is given and fixed), $r$ is the regularization to be specified below, and $\xbf$ is the (gray-scale) image with $n$ pixels to be reconstructed.
To instantiate our derivation below, we use the linear least squares data fidelity term $f(\xbf) = (1/2)\cdot \|\Phi \xbf - \zbf\|^2$, where $\Phi \in \Re^{n' \times n}$ and $\zbf \in \Re^{n'}$ are given. However, as can be seen from our derivation below, $f$ can be any given smooth but nonconvex function with Lipschitz continuous gradient $\nabla f$.
Here $\|\xbf\|$ denotes the standard 2-norm of of a vector $\xbf$, and $\|\Phi\|$ stands for the induced 2-norm of a matrix $\Phi$.
In this paper, we would also like to leverage the robust shrinkage threshold operator in computer vision and image processing in the regularization $r$.
More specifically, we parametrize the regularization term $r$ as the $(2,1)$-norm of $g(\xbf)$, where $g=(g_1,\dots,g_m)$ with $g_i: \Re^{n} \to \Re^{d}$ for $i=1,\dots,m$ is a smooth nonlinear (with possibly nonconvex components) operator to be learned later:
\begin{equation}\label{eq:r}
r(\xbf)=\|g(\xbf)\|_{2,1} = \sum_{i = 1}^{m} \|g_i(\xbf)\|,
\end{equation}
where $g_i(\xbf) = ([g_i(\xbf)]_1,\cdots,[g_i(\xbf)]_d) \in \Re^{d}$, and $[g_i(\xbf)]_j \in \Re$ is the $j$th component (channel) of $g_i(\xbf)$ for $j=1,\dots,d$.
Here $m$ can be different from $n$ if the result $g(\xbf)$ changes the size of $\xbf$.
As we can see later, the $(2,1)$-norm in $r$ yields the soft shrinkage operation on $(g_1,\dots,g_m)$, which plays the role of a robust nonlinear activation function in the deep network architecture later.
The nonlinear operator $g$, on the other hand, is an adaptive sparse feature extractor learned from training data.
However, it is also worth noting that the derivation and convergence analysis below can also be applied to \eqref{eq:F} with general nonsmooth and nonconvex regularization $r$.
%


\subsection{Smooth Approximation of Nonsmooth Regularization}
\label{subsec:smooth}
To tackle the nonsmooth and nonconvex regularization term $r(\xbf)$ in \eqref{eq:r}, we first employ Nesterov's smoothing technique for convex function \cite{nesterov2005smooth} to smooth the (2,1)-norm part of $r(\xbf)$ (the nonlinear and nonconvex term $g$ remains untouched).
To this end, we first apply the dual form of (2,1)-norm in $r(\xbf)$ as follows:
\begin{equation}\label{eq:Rx}
r(\xbf) = \max_{\ybf \in Y}\ \langle g(\xbf), \ybf \rangle,
\end{equation}
where $\ybf \in Y$ is the dual variable, $Y$ is the dual space defined by
\[
Y := \cbr[1]{ \ybf=(\ybf_1,\dots,\ybf_m) \in \Re^{md}\ | \ \ybf_i=(y_{i1},\dots,y_{id}) \in \Re^{d}, \| \ybf_i \| \leq 1, \forall\, i }.
\]
For any $\eta>0$, we consider the smooth version $\reta$ of $r$ by perturbing the dual form \eqref{eq:Rx} as follows:
\begin{equation}\label{eq:r_eta}
\reta(\xbf) = \max_{\ybf \in Y} \ \langle g(\xbf), \ybf \rangle - \frac{\eta}{2} \|\ybf\|^2,
\end{equation}
Then one can readily show that
\begin{equation}\label{relation}
\reta(\xbf)\leq r(\xbf) \leq \reta(\xbf) +\frac{m\eta}{2},\quad \forall\, \xbf \in \Re^n.
\end{equation}
Note that the perturbed dual form in \eqref{eq:r_eta} has closed form solution: denoting
\begin{equation}\label{eq-7-1}
\ybf_\eta^* = \argmax_{\ybf \in Y}\ \langle g(\xbf) , \ybf \rangle - \frac{\eta}{2} \|{\ybf}\|^2,
\end{equation}
%
then solving \eqref{eq-7-1}, we obtain the closed form of $\ybf_\eta^*=([\ybf_\eta^*]_1,\dots,[\ybf_\eta^*]_m)$ with
\begin{equation}\label{eq-7}
[\ybf_\eta^*]_i =
\begin{cases}
\frac{1}{\eta} g_i(\xbf), & \mbox{if} \ \|g_i(\xbf) \| \leq \eta, \\
\frac{g_i(\xbf)}{\|g_i(\xbf)\|}, & \mbox{otherwise},
\end{cases}
\end{equation}
for $i=1,\dots,m$.
Plugging \eqref{eq-7} back into \eqref{eq:r_eta}, we have

\begin{equation}\label{eq-8}
 \reta(\xbf)  = \sum_{i \in I_1} \frac{1}{2\eta}  \|g_i(\xbf)\|^2 + \sum_{i \in I_2}(\|g_i(\xbf)\| - \frac{\eta}{2}),
\end{equation}
where $I_1 = \{ i \in [m] \ | \ \|g_i(\xbf)\| \leq \eta \}$, $I_2 = [m] \setminus I_1$, and $[m]:=\{1,\dots,m\}$.
Moreover, it is easy to show from \eqref{eq-8} that
\begin{equation}\label{eq-9}
\nabla \reta(\xbf) =  \sum_{i \in I_1} \frac{1}{\eta}  g_i(\xbf) \nabla g_i(\xbf)  + \sum_{i \in I_2}  \frac{g_i(\xbf)}{\|g_i(\xbf)\|} \nabla g_i(\xbf),
\end{equation}
where $\nabla g_i(\xbf)$ is the Jacobian of $g_i$ at $\xbf$.

The smoothing technique above allows us to approximate the nonsmooth function with rigorous convergence and iteration complexity analysis of our novel gradient descent algorithm for the original nonsmooth nonconvex problem \eqref{eq:F}.

\subsection{A Novel Gradient Descent Type Algorithm} \label{Res-GD}
In this subsection, we propose a novel gradient descent type algorithm for solving the minimization problem \eqref{eq:F} with smoothed regularization $\reta$ in \eqref{eq:r_eta}.
To employ the effective residual network structure \cite{ResNet} in its mapped network later, we need to incorporate the corresponding feature in our algorithmic design here.
To this end, we consider the objective function $F_\eta$ with $\reta$ as follows:
\begin{equation}\label{eq:F_eta}
F_\eta(\xbf):=f(\xbf) + \reta (\xbf).
\end{equation}
Note that, unlike $F$ in \eqref{eq:F}, $F_\eta$ is nonconvex but smooth due to the existence of gradient $\nabla \reta$ in \eqref{eq-9}.

Now we are ready to present our residual gradient descent (ResGD) algorithm.
In the $k$th iteration, we first compute
\begin{equation}
    \label{eq:b}
    \bbf^k = \xbf^k - \alpha_k \nabla f(\xbf^k),
\end{equation}
where $\alpha_k$ is the step size to be specified later.
We then compute two candidates, denoted by $\ubf^{k+1}$ and $\vbf^{k+1}$, for the next iterate $\xbf^{k+1}$ as follows:
\begin{subequations}
\begin{align}
\ubf^{k+1} &= \argmin_{\xbf}\ \langle \nabla f(\xbf^k) , \xbf - \xbf^k\rangle + \frac{1}{2\alpha_k} \| \xbf - \xbf ^k \| ^2 + \langle \nabla \reta(\bbf^k) , \xbf - \bbf^k\rangle \label{eq:u} \\ &\qquad\qquad\qquad   +  \frac{1}{2\beta_k} \| \xbf - \bbf ^k \| ^2, \nonumber \\
\vbf^{k+1} &= \argmin_{\xbf}\ \langle \nabla f(\xbf^k) , \xbf - \xbf^k\rangle + \langle \nabla \reta(\xbf^k) , \xbf - \xbf^k\rangle  + \frac{1}{2\alpha_k} \| \xbf - \xbf ^k \| ^2, \label{eq:v}
\end{align}
\end{subequations}
where $\beta_k$ is another step size along with $\alpha_k$.
Note that both minimization problems in \eqref{eq:u} and \eqref{eq:v} have closed form solutions:
\begin{subequations}
\begin{align}
\ubf^{k+1} &= \bbf^{k} - \gamma_k \nabla \reta(\bbf^{k}) \label{eq:u_closed} \\
\vbf^{k+1} &= \bbf^{k} - \alpha_k \nabla \reta(\xbf^{k}) \label{eq:v_closed}
\end{align}
\end{subequations}
where $\nabla \reta$ is defined in \eqref{eq-9}, and $\gamma_k = \frac{\alpha_k \beta_k}{\alpha_k + \beta_k}$.
Then we choose between $\ubf^{k+1}$ and $\vbf^{k+1}$ that has the smaller function value $F_\eta$ to be the next iterate $\xbf^{k+1}$:
\begin{equation}\label{eq:x}
    \xbf ^{k+1} =
\begin{cases}
\ubf^{k+1} &\text{if $F_\eta(\ubf^{k+1}) \leq F_\eta(\vbf^{k+1})$}, \\
\vbf^{k+1} & \text{otherwise}.
\end{cases}
\end{equation}
This algorithm is summarized in Algorithm \ref{alg1}.
If the $\ubf$-step is disabled, then Algorithm \ref{alg1} Res-GD reduces to the standard gradient descent method for $F_\eta$ in \eqref{eq:F_eta}.
However, this $\ubf$-step corresponds to a residual network structure in the ResGD-Net we construct later, and it is critical to improving the practical performance of ResGD-Net.
%
%
\begin{algorithm}[t]
\caption{Residual Gradient Descent Algorithm (Res-GD)}
\label{alg1}
\begin{algorithmic} 
\STATE \textbf{Input:} Initialization $\xbf^0$.
\STATE \textbf{Output:} $\xbf=\xbf^{K}$.
\FOR{$k=1,2,\dots,K$}
\STATE $\bbf \gets \xbf - \alpha_k \nabla f(\xbf)$.
\STATE $\ubf \gets \bbf - \gamma_k \nabla \reta(\bbf)$.
\STATE $\vbf \gets \bbf - \alpha_k \nabla \reta(\xbf)$.
\STATE If $F_\eta(\ubf) \le F_\eta(\vbf)$, $\xbf \gets \ubf$; Otherwise, $\xbf \gets \vbf$.
\ENDFOR
\end{algorithmic}
\end{algorithm}

\subsection{Convergence and Complexity Analysis}
\label{subsec:convergence}
In this subsection, we provide a comprehensive convergence analysis with iteration complexity of the proposed Algorithm \ref{alg1} Res-GD.
To this end, we need several mild assumptions on the functions involved in Algorithm \ref{alg1}.
More specifically, we have Assumptions (A1) and (A2) on the smooth nonlinear operator $g$ in the regularization function $r$ in \eqref{eq:r}, (A3) on the function $f$, and (A4) on the objective function $F$ in \eqref{eq:F}, as follows.
\begin{assumption}[A\ref{hyp:first}] \label{hyp:first}
%
%
The operator $g(\xbf)$ is continuously differentiable with $L_g$-Lipschitz gradient $\nabla g(\xbf)$, i.e., there exists a constant $L_g>0$, such that $\|  \nabla g(\xbf_1)-\nabla g(\xbf_2)\| \leq L_g \|  \xbf_1-\xbf_2\| $ for all $\xbf_1$, $\xbf_2\in\Re^{n}$.
\end{assumption}
\begin{assumption}[A\ref{hyp:second}] \label{hyp:second}
$\ \sup_{\xbf} \| \nabla g(\xbf)\| \leq M $ for some constant $M>0$.
\end{assumption}

\begin{assumption}[A\ref{hyp:third}] \label{hyp:third}
The function $f(\xbf)$ is continuously differentiable with $L_f$-Lipschitz gradient $\nabla f(\xbf)$.
\end{assumption}

\begin{assumption}[A\ref{hyp:fourth}] \label{hyp:fourth}
$F(\xbf)$ is coercive, i.e. $F(\xbf)\rightarrow \infty$ as
$\|\xbf\|\rightarrow \infty$.
\end{assumption}




Due to non-differentiable regularization function in \eqref{eq:F}, we cannot directly consider stationary points in the classical sense.
%
%
Therefore, we consider the following constrained minimization equivalent to \eqref{eq:F}:
\begin{subequations}
\label{eq:reformulate}
\begin{align}
\min_{\xbf,\,\ybf}\quad & f(\xbf)+\sum_{i=1}^{m} y_i\label{reformulate-obj} \\
\mbox{subject to}\quad & y_i^2 \geq \|  g_i(\xbf)\| ^2,\quad i=1,...,m,\label{eq:reformulate-const_1} \\
& y_i\geq 0, \quad i=1,...,m.\label{eq:reformulate-const_2}
\end{align}
\end{subequations}
where $\ybf=(y_1,\dots,y_m) \in \Re^m$.
To see the equivalence between \eqref{eq:F} and \eqref{eq:reformulate}, we observe that, for any fixed $\xbf$, the optimal $\ybf$ ensures that $y_i^2=\|  g_i(\xbf)\| _2^2$, and thus, $y_i=\|  g_i(\xbf)\| _2$ (c.f., $y_i\geq 0$) for all $i=1,...,m$.
Then $(\xbf^*,\ybf^*)$ is called a Karush-Kuhn-Tucker (KKT) point of \eqref{eq:reformulate} if the following conditions are satisfied:
\begin{subequations}
\label{eq:KKT}
\begin{align}
&\nabla f(\xbf^*)+2\sum_{i=1}^{m}\mu_i g_i(\xbf^*) \nabla g_i(\xbf^*) =0\label{stationarity_1}
\\&
  1-2\mu_i y_i^*-\lambda_i =0,\quad i=1,...,m\label{stationarity_2}
\\& \mu_i [\|  g_i(\xbf^*)\| ^2-(y_i^*)^2]=0,\quad i=1,...,m\label{eq:complementarity_1}
\\& \lambda_i y_i^* =0,\quad i=1,...,m\label{eq:complementarity_2}
\\& \lambda_i,\,\mu_i\geq 0,\quad i=1,...,m\label{eq:dual_feasibility}
\\& y_i^2\geq \|  g_i(\xbf)\| ^2,\quad y_i\geq 0,\quad i=1,...,m.\label{eq:primal_feasibility}
\end{align}
\end{subequations}
%
%
for some $\lambda_i,\mu_i \in \Re$, $i=1,...,m$.
Here $\mu_i$ and $\lambda_i$ are the Lagrangian multipliers associated with the constraints \eqref{eq:reformulate-const_1} and \eqref{eq:reformulate-const_2}, respectively.
In particular, \eqref{stationarity_1}-\eqref{stationarity_2} are stationarity, \eqref{eq:complementarity_1}-\eqref{eq:complementarity_2} are complementary slackness, and \eqref{eq:dual_feasibility} and \eqref{eq:primal_feasibility} stem from dual and primal feasibility, respectively.
To measure the closeness of an approximation generated by Algorithm \ref{alg1}, we propose to generalize the definition above to the $\epsilon$-KKT point as follows.
\begin{definition}\label{epsilon-KKT}
For any $\epsilon\geq 0$, $\xbf^*_{\epsilon}$ is called an $\epsilon$-KKT solution to \eqref{eq:F} if there exist $(\mu_i,\lambda_i,\,y_i)$, $i=1,...,m$, such that
\begin{subequations}
\label{eq:eps-KKT}
\begin{align}
&\big\| \nabla f(\xbf^*_{\epsilon})+2\sum_{i=1}^K\mu_i g_i(\xbf^*_{\epsilon}) \nabla g_i(\xbf^*_{\epsilon})\big\| \leq \epsilon \label{Con_1}
\\&
  1-2\mu_i y_i-\lambda_i =0,\quad i=1,...,m \label{Con_1.5}
\\& |\mu_i( \|  g_i(\xbf^*_{\epsilon})\| ^2-y_i^2) | \leq \epsilon,\quad i=1,...,m; \label{Con_2}
\\& \lambda_i y_i =0,\quad i=1,...,m\label{Con_3}
\\& \lambda_i,\,\mu_i\geq 0,\quad i=1,...,m \label{Con_4}
\\& y_i\geq \|  g_i(\xbf^*_{\epsilon})\| ,\quad i=1,...,m. \label{Con_5}
\end{align}
\end{subequations}
\end{definition}
In this definition, \eqref{Con_1}--\eqref{Con_4} correspond to the $\epsilon$-approximation to \eqref{stationarity_1}--\eqref{eq:dual_feasibility} and \eqref{Con_5} is derived from the primal feasibility.
%

Our goal is then to study the convergence of the proposed algorithm and its iteration complexity to obtain an $\epsilon$-KKT solution to \eqref{eq:F} in the sense of Definition \ref{epsilon-KKT}.
To this end, we first need the following lemma to characterize the Lipschitz constant for $\nabla \reta$.
\begin{lemma}\label{Lipschitz_constant_eta}
Under Assumptions (A\ref{hyp:first}) and (A\ref{hyp:second}), the gradient $\nabla \reta$ of the smoothed function $\reta$ defined in \eqref{eq:r_eta} is Lipschitz continuous with constant $mL_g+\frac{M^2}{\eta}$.
\end{lemma}

\begin{proof}
We first define $\ybf_1$ and $\ybf_2$ as follows,
\begin{align*}
\ybf_1 &=\argmax_{\ybf\in Y}\ \langle g(\xbf_1),\,\ybf\rangle -\frac{\eta}{2} \| \ybf \|^2, \\
\ybf_2 &=\argmax_{\ybf\in Y}\ \langle g(\xbf_2),\,\ybf\rangle -\frac{\eta}{2} \| \ybf \|^2.
\end{align*}
Due to the concavity of the problems above (in $\ybf$) and the optimality conditions of $\ybf_1$ and $\ybf_2$, we have
\begin{align}
&\langle g(\xbf_1)-\eta \ybf_1,\,\ybf_2-\ybf_1\rangle\leq 0;
\\&
\langle g(\xbf_2)-\eta \ybf_2,\,\ybf_1-\ybf_2\rangle\leq 0.
\end{align}
Adding the two inequalities above yields
\begin{align}
\langle g(\xbf_1)-g(\xbf_2)-\eta \left(\ybf_1-\ybf_2\right),\,\ybf_2-\ybf_1\rangle\leq 0,
\end{align}
which, together with the Cauchy-Schwarz inequality, implies
\begin{equation*}
    \|  g(\xbf_1)-g(\xbf_2) \|  \cdot  \|  \ybf_1-\ybf_2 \| \geq \langle g(\xbf_1)-g(\xbf_2),\,\ybf_1-\ybf_2\rangle\geq \eta\, \|  \ybf_2-\ybf_1 \| ^2.
\end{equation*}
Therefore, $ \|  g(\xbf_1)-g(\xbf_2) \| \geq \eta\, \|  \ybf_1-\ybf_2 \| $.
Following the notations in Section 3.1, we have $\nabla r_\eta(\xbf)=\nabla g(\xbf)^\top \ybf^*$ where $\ybf^*=\argmax_{\ybf\in Y} \langle g(\xbf),\,\ybf\rangle-\frac{\eta}{2} \| \ybf \| ^2 = \argmin_{\ybf\in Y} \frac{\eta}{2}\|\ybf - \eta^{-1} g(\xbf)\|^2$ and $Y=\{\ybf \in \mathbb{R}^{md}\ | \ \|\ybf_i\| \le 1,\ 1\le i \le m\}$.
Therefore, the optimality of $\ybf_1$ and $\ybf_2$ above implies
\begin{align*}
& \| \nabla r_{\eta}(x_1) - \nabla r_{\eta}(x_2) \| = \left \| \nabla g(\xbf_1)^\top \ybf_1-\nabla g(\xbf_2)^\top \ybf_2\right \| \\
= & \left \| \left(\nabla g(\xbf_1)^\top \ybf_1-\nabla g(\xbf_2)^\top \ybf_1\right)+\left(\nabla g(\xbf_2)^\top \ybf_1-\nabla g(\xbf_2)^\top \ybf_2\right)\right \| \\
\le & \left \| \left(\nabla g(\xbf_1)  -\nabla g(\xbf_2) \right)^\top \ybf_1 \right \| +  \|  \nabla g(\xbf_2)  \|  \left \|   \ybf_1-  \ybf_2 \right \|
\\
\le & \left \|  \nabla g(\xbf_1)  -\nabla g(\xbf_2)  \right \|  \cdot \| \ybf_1   \| +  \frac{1}{\eta}\cdot  \|  \nabla g(\xbf_2)  \| \cdot  \|  g(\xbf_1)-g(\xbf_2) \| .
\end{align*}
Recalling the assumptions of (A1) and (A2), we have $ \| \nabla g(\xbf) \| \leq M$ for all $\xbf\in\Re^n$ and that $\nabla g(\xbf)$ is Lipschitz with constant $L_g$.
Since $\max_{\ybf\in Y} \|  \ybf  \| \leq \sqrt{m}$, we have
\[
\left \| \nabla g(\xbf_1)^\top \ybf_1-\nabla g(\xbf_2)^\top \ybf_2\right \| \leq \left(\sqrt{m}\cdot L_g+\frac{M^2}{\eta}\right)\,  \|  \xbf_1-\xbf_2 \|,
\]
which completes the proof.
%
\end{proof}
Our main results on the convergence and iteration complexity of Algorithm \ref{alg1} ResGD are summarized in the following theorem.

\begin{theorem} Assume (A1)--(A4) hold. For any initial $\xbf^0$ and constants $\alpha>\beta>1$, the following statements hold for the sequences $\{\xbf^k\}$ and $\{\vbf^k\}$ generated by Algorithm 1 with $(\alpha L_{\eta})^{-1}\leq \alpha_k \leq (\beta L_{\eta})^{-1}$ and $L_\eta:=L_f+m L_g+\frac{M^2}{\eta}$:
\begin{enumerate}
\item The sequence $\{\xbf^k\}$ is bounded. The function $F_\eta$ takes the same value, denoted by $F_\eta^*$, at all accumulation points of $\{\xbf^k\}$. Moreover, for any accumulation point $\xbf^*$, there is
\begin{equation}\label{GFx*}
\nabla F_{\eta}(\xbf^*)=0.
\end{equation}
\item
For any $\epsilon>0$, there exists $k \le \lfloor \frac{2 \alpha^2 L_{\eta}(F(\xbf^{0})-F^*_{\eta})}{(\beta-1)\epsilon^2} \rfloor + 1$ such that
\begin{equation}\label{GFxK}
\|\nabla F_{\eta}(\xbf^{k})\| \leq \epsilon.
\end{equation}
\item For any $\epsilon>0$, let $\eta=\epsilon$, then there exists $k \le \lfloor \frac{2(F_\eta(\xbf^{0})-F_\eta^*)\alpha^2 (L_f+mL_g +M^2/\epsilon)}{(\beta-1)\epsilon^2} \rfloor + 1=O(\epsilon^{-3})$,  such that $\xbf^k$ is an $\epsilon$-KKT solution to \eqref{eq:F} in the sense of Definition \ref{epsilon-KKT}.
\end{enumerate}
\end{theorem}

\begin{proof}
\textit{1.} Due to the optimality condition of $\vbf^{k+1}$ in the algorithm, we have
\begin{equation}\label{eqv}
    \langle\nabla F_{\eta}(\xbf^k), \vbf^{k+1}-\xbf^k\rangle+\frac{1}{2\alpha_k} \|  \vbf^{k+1}-\xbf^k \| ^2 \leq 0.
\end{equation}
Due to both of (A4) and Lemma 1 (under the assumptions of (A1) and (A2)), we know that $F_{\eta}(\xbf)$ has $L_\eta$-Lipschitz continuous gradient, where $L_\eta:=L_f+\sqrt{m} L_g+\frac{M^2}{\eta}$, which implies that
\begin{equation}\label{lipF}
  F_{\eta}(\vbf^{k+1}) \leq F_{\eta}(\xbf^k) + \langle\nabla F_{\eta}(\xbf^k), \vbf^{k+1}-\xbf^k\rangle+\frac{L_{\eta}}{2} \|  \vbf^{k+1}-\xbf^k \| ^2.
\end{equation}
Combining $\eqref{eqv}$, $\eqref{lipF}$ and $\alpha_k \le (\beta L_{\eta})^{-1}$ with $\beta>1$ yields
\begin{equation}\label{diff}
F_{\eta}(\vbf^{k+1})-F_{\eta}(\xbf^k)\leq - (\frac{1}{2\alpha_k}-\frac{L_{\eta}}{2}) \|  \vbf^{k+1}-\xbf^k \|^2 \le - \frac{(\beta-1) L_{\eta}}{2} \|  \vbf^{k+1}-\xbf^k \|^2.
\end{equation}
If $F_{\eta}(\ubf^{k+1})\leq F_{\eta}(\vbf^{k+1})$, then
$\xbf^{k+1}=\ubf^{k+1}$, and $F_{\eta}(\xbf^{k+1})=F_{\eta}(\ubf^{k+1})\leq F_{\eta}(\vbf^{k+1})$. If $F_{\eta}(\vbf^{k+1})< F_{\eta}(\ubf^{k+1})$, then
$\xbf^{k+1}=\vbf^{k+1}$, and $F_{\eta}(\xbf^{k+1})=F_{\eta}(\vbf^{k+1})$.
Therefore, in either case, $\eqref{diff}$ implies
\begin{equation}\label{recursive}
   F_{\eta}(\xbf^{k+1})\leq F_{\eta}(\vbf^{k+1}) \leq F_{\eta}(\xbf^{k})\leq \ldots \leq F_{\eta}(\xbf^{0}).
\end{equation}
for all $k\ge 0$.

Since $F(\xbf)$ is coercive, from $ r_\eta(\xbf)\leq r(\xbf) \leq  r_\eta(\xbf) +\frac{m\eta}{2}$, we know $F_{\eta}(\xbf)$ is also coercive.
Therefore, $\{\xbf^k\}$ and $\{\vbf^k\}$ are bounded, and hence $\{\xbf^k\}$ has at least one accumulation point.
Moreover, $\{F_{\eta}(\xbf^{k})\}$ is non-increasing due to \eqref{recursive} and bounded below, which means that $\{F_{\eta}(\xbf^{k})\}$ is a convergent (numerical) sequence.
Denote the limit of $\{F_{\eta}(\xbf^{k})\}$ by $F^*_\eta$.
Let $\xbf^*$ be any accumulation point of $\{\xbf^k\}$, i.e., there exists a subsequence $\{\xbf^{k_j}\}$ of $\{\xbf^k\}$, such that $\xbf^{k_j} \to \xbf^*$ as $j\to \infty$.
Then the continuity of $F_{\eta}(\xbf)$ implies that $F_{\eta}(\xbf^{k_j}) \to F_{\eta}(\xbf^*)$ as $j\to\infty$. Since $F_{\eta}(\xbf^{k_j})$ is a subsequence of the convergent sequence $F_{\eta}(\xbf^{k})$ which has limit $F^*_\eta$, we know $F_{\eta}(\xbf^*) = F^*_\eta$.
Note that $\xbf^*$ is an arbitrary accumulation point, therefore every accumulation point of $\{\xbf^{k}\}$ has the same function value $F^*_\eta$.

Summing up \eqref{diff} with respect to $k \ge 0$ and noting that $F_{\eta}(\xbf^k)\downarrow F^*_\eta=F_{\eta}(\xbf^*)$, we know that, with $\alpha_k \leq (\beta L)^{-1}$, there is
\begin{equation}\label{sum v-x}
    \sum_{k=0}^{\infty} \|  \vbf^{k+1}-\xbf^k \| ^2 \leq \frac{2(F_{\eta}(\xbf^0)-F_{\eta}(\xbf^*))}{(\beta-1) L_{\eta}} < \infty.
\end{equation}
Hence there is
\begin{equation} \label{norm v-x}
     \| \vbf^{k+1}-\xbf^k \| ^2 \rightarrow 0,\quad \mbox{as}\quad k\to\infty.
\end{equation}
From the optimality condition of $\vbf^{k+1}$, we have
\begin{equation} \label{GFv-x}
   \nabla F_{\eta}(\xbf^k)=\frac{\xbf^k-\vbf^{k+1}}{\alpha_k}.
\end{equation}
Combining \eqref{norm v-x} and \eqref{GFv-x}, and substituting $\xbf^k$ by any of its convergent subsequence $\{\xbf^{k_j}\}$ with limit $\xbf^*$ as above (also the corresponding subsequence of $\vbf^k$), we obtain
$\|\nabla F_{\eta}(\xbf^{k_j})\| \rightarrow 0$. Then from the continuity of $\nabla F_{\eta}$, we obtain $\nabla F_\eta (\xbf^*)=0$. This proves the first statement.
%


\textit{2.}
Since $(\alpha L_{\eta})^{-1}\leq \alpha_k \leq (\beta L_{\eta})^{-1}$ for some $\alpha >\beta>1$, \eqref{norm v-x} implies that there exists $K^*:=\min\{k: \|  \vbf^{k+1}-\xbf^k \| \leq (\alpha L_{\eta})^{-1}\epsilon \} < \infty$.
Note that $ \|  \vbf^{ k+1} -\xbf^{ k} \| ^2 \geq (\alpha L_{\eta})^{-2} \epsilon^2$ for all $k< K^*$.
Therefore \eqref{diff} implies that $F_{\eta}(\vbf^{k+1})-F_{\eta}(\xbf^k) \le-(\beta-1)\epsilon^2/(2\alpha^2 L_{\eta})$ for all $k< K^*$.
From \eqref{recursive} and the fact that $F_{\eta}(\xbf^k)\downarrow F^*_\eta=F_{\eta}(\xbf^*)$, we get
\begin{align*}
   0 &\leq F_{\eta}(\xbf^{K^*})-F_{\eta}(\xbf^*)=F_{\eta}(\xbf^0)-F_{\eta}(\xbf^*)+ \sum_{k=0}^{K^*-1}\left[
F_{\eta}(\xbf^{k+1})-F_{\eta}(\xbf^k)\right]\\
&\leq
-\frac{(\beta-1)\epsilon^2}{2 \alpha^2 L_\eta}\cdot K^*+F_{\eta}(\xbf^0)-F_{\eta}^*.
\end{align*}
Therefore,   $K^* \le \frac{2 \alpha^2 L_{\eta}(F(\xbf^{0})-F^*_{\eta})}{(\beta-1)\epsilon^2}$.
Moreover, by the definition of $K^*$, we have that
\[
\frac{\| \vbf^{K^*+1} -\xbf^{K^*} \|}{\alpha_{K^*}} \leq \frac{\epsilon}{\alpha_{K^*}\alpha L_\eta} \le  \epsilon.
\]
Therefore,
$ \|  \nabla F_\eta(\xbf^{K^*}) \| = \frac{1}{\alpha_{K^*}}\| \vbf^{{K^*}+1}-\xbf^{K^*} \| \leq \epsilon$.
Setting $k=K^*$ proves the claim.

\textit{3.}
To prove the last statement, we first show that for $\eta = \epsilon$, $\hat{\xbf} $ is an $\epsilon$-KKT solution to the original problem with nonsmooth $F$ as objective function  provided that $\|\nabla F_{\eta}(\hat{\xbf})\| \leq \epsilon$.
To this end, we note that
\begin{align}
\label{eq:KKT_solution_smoothed}
\nabla F_\eta(\hat{\xbf})=\nabla f(\hat{\xbf})+\sum_{i\in I_1}\frac{1}{\eta}g_i(\hat{\xbf})\nabla g_i(\hat{\xbf}) +\sum_{i\in I_2}\frac{g_i(\hat{\xbf})}{ \|  g_i(\hat{\xbf}) \| }\nabla g_i(\hat{\xbf}),
\end{align}
where $I_1=\{i\ | \ \|g_i(\hat{\xbf})\| \le \eta\}$ and $I_2=\{i\ | \ \|g_i(\hat{\xbf})\| > \eta\}$.
%
%
By setting $y_i=\max\{\eta,\, \|  g_i(\hat{\xbf}) \| \}$, $\mu_i=\frac{1}{2 \|  g_i(\hat{\xbf}) \| }$ if $\|  g_i(\hat{\xbf}) \| > \eta$ and $\frac{1}{2\eta}$ {otherwise},  and $\lambda_i=0$, for all $i=1,...,K$, we can easily verify that all the $\epsilon$-KKT conditions are satisfied at $\hat{\xbf}$ provided $\|\nabla F_{\eta}(\hat{\xbf})\| \leq \epsilon$.
Note that $\|\nabla F_{\eta}(\xbf^{K^*})\| \leq \epsilon$, we know $\xbf^{K^*}$ is an $\epsilon$-KKT solution to the original problem.

Furthermore, because $\eta= \epsilon$, we have $L_{\eta}\leq L_f+\sqrt{m}L_g +M^2/\epsilon$.
Then, for $(\alpha L_{\eta})^{-1}\leq \alpha_k \leq (\beta L_{\eta})^{-1}$, we have
\[
K^* \leq \frac{2\alpha^2 L_{\eta}(F_\eta(\xbf^{0})-F_\eta^*)}{ (\beta-1)\epsilon^2} \le \frac{2\alpha^2 (F_\eta(\xbf^{0})-F_\eta^*) (L_f+\sqrt{m}L_g +M^2/\epsilon)}{ (\beta-1)\epsilon^2}=O(\epsilon^{-3}).
\]
Setting $k=K^*$ proves the claim.
This completes the proof.
\end{proof}

\subsection{Residual Gradient Descent Network} \label{Network}

In this subsection, we construct a deep neural network imitating the proposed Algorithm \ref{alg1} with nonlinear function $g$ to be trained from data.
We first parametrize the function $g(\xbf)$ as a convolutional network as follows:
\begin{equation}\label{AsigmaB}
g(\xbf) = B \sigma (A\xbf),
\end{equation}
where $A \in \Re ^ {md \times n}$ and $B \in \Re^{md \times md}$ are the matrix representation of two convolution operations.
%
%
For example, to obtain a relative larger receptive field \cite{LeB17} for image reconstruction, we design $A$ to be a cascade of two convolutions, where the first convolution is with $d$ kernels of size $3\times 3$ and the second with $d$ kernels of size $3\times 3\times d$.
Besides, $B$ corresponds to convolution with $d$ kernels of size $3\times 3\times d$.
Here, $\sigma$ represents a component-wise activation function.
In this paper, we use the following smooth nonlinear activation $\sigma \in \mathcal{C}^1$:

%
%
%
\begin{equation}\label{eq-activation}
\sigma (x) =
\begin{cases}
0, & \mbox{if} \ x \leq -\delta, \\
\frac{1}{4\delta} x^2 + \frac{1}{2} x + \frac{\delta}{4}, & \mbox{if} \ -\delta < x < \delta, \\
x, & \mbox{if} \ x \geq \delta.
\end{cases}
\end{equation}
%
Here $\delta>0$ is a prescribed threshold (set to $0.1$ in our experiment).
Note that $g$ defined in \eqref{AsigmaB} satisfies both assumptions (A\ref{hyp:first})--(A\ref{hyp:fourth}) in Section \ref{subsec:convergence}.
%
%
From \eqref{eq-9} and \eqref{AsigmaB}, we have $g_i(\xbf)=(B\sigma A \xbf)_i$ and hence
\begin{equation}\label{eq-11}
 \nabla \reta(\xbf) = A^{\top} \sigma' (A \xbf) B^{\top} \del[2]{ \sum_{i \in I_1} \frac{(B\sigma A \xbf)_i}{\eta} + \sum_{i \in I_2} \frac{(B\sigma A \xbf)_i}{\|(B\sigma A \xbf)_i\|} },
\end{equation}
where $I_1 = \{ i \in [m] \ | \ \|(B\sigma A \xbf)_i\| \leq \eta \}$, $I_2 = \{ i\in [m] \ | \ \|(B\sigma A \xbf)_i\| > \eta \}$.

%

The detailed updating scheme of each phase of the proposed network is depicted in Fig. \ref{fig-2}.
Specifically, we prescribe the iteration number $K$, which is also the phase number of the proposed ResGD-Net.
We enable the step sizes $\alpha_k$ and $\gamma_k$ to vary in different phases, moreover, all $\{\alpha_k, \gamma_k\}_{k = 1} ^ K$ and threshold $\eta$ are designed to be learnable parameters fitted by data.
To further increase the capacity of the proposed network, we employ the learnable inverse operator. More precisely, we replace $A^{\top}$ and $B^{\top}$ in \eqref{eq-11} by learnable operators $\widetilde{A}\in \Re ^ {n \times md}$ and $\widetilde{B}\in \Re ^ {md \times md}$. To approximately achieve $\widetilde{A} \approx A^{\top}$ and $\widetilde{B} \approx B^{\top}$,
we incorporate the constraint term $ \mathcal{L}_{constraint} = \| \widetilde{A} - A^{\top} \| ^2_F + \| \widetilde{B} - B^{\top} \|^2_F$ to the loss function during training to acquire the data-driven inverse operators, where $\|\cdot\|_F$ is the Frobenius norm.
In addition, $\widetilde{A}$ is implemented as
a cascade of two transposed convolutional operators \cite{dumoulin2016guide} and $\widetilde{B}$ as one transposed convolutional operator, similar to $A$ and $B$.
\begin{figure}[t]
\vspace{-2pt}
\begin{center}
   \includegraphics[width=1.0\linewidth]{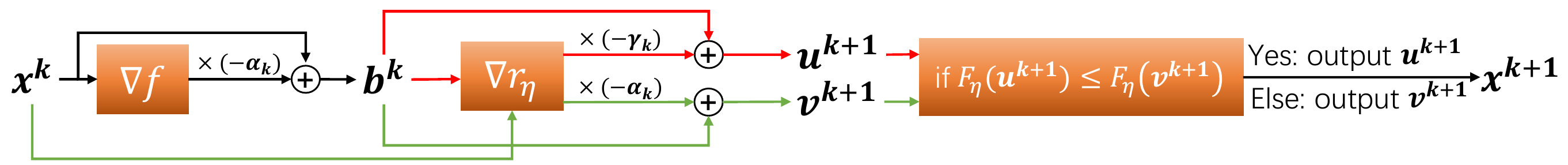}
\end{center}
\vspace{-7pt}
\caption{Illustration of the $k$th phase of ResGD-Net. The red and green arrows represent the updating for $\ubf ^{k+1}$ (Eq. \eqref{eq:u_closed}) and $\vbf ^{k+1}$ (Eq. \eqref{eq:v_closed}) respectively}
\vspace{-8pt}
\label{fig-2}
\end{figure}
\vspace{-5pt}
\subsubsection{Network Training:}
We denote $\Theta$
to be the set of all learnable parameters of the proposed ResGD-Net which consists of the weights of the convolutional operators $\{A, B\}$ and transposed convolutional operators $\{\widetilde{A}, \widetilde{B}\}$, step sizes $\{\alpha_k, \gamma_k\}_{k = 1} ^ K$ and threshold $\eta$. Given $N$ training data pairs $\{(\zbf^{(i)}, \xbf^{(i)})\}_{i=1}^{N}$, where each $\xbf^{(i)}$ is the ground truth data and $\zbf^{(i)}$ is the measurement of $\xbf^{(i)}$, the loss function $\mathcal{L}(\Theta)$ is defined to be the sum of the discrepancy loss $\mathcal{L}_{discrepancy}$ and the constraint loss $\mathcal{L}_{constraint}$:
\begin{equation}\label{loss}
  \mathcal{L}(\Theta) = \underbrace{\frac{1}{N}\sum_{i = 1}^N \|\xbf^K(\zbf^{(i)};\Theta) - \xbf^{(i)}\|^2}_{\mathcal{L}_{discrepancy}} + \vartheta \underbrace{ \{ \| \widetilde{A} - A^{\top} \| ^2_F + \| \widetilde{B} - B^{\top} \| ^2_F \} }_{\mathcal{L}_{constraint}},
\end{equation}
where $\mathcal{L}_{discrepancy}$ measures the discrepancy between the ground truth $\xbf^{(i)}$ and $\xbf^K(\zbf^{(i)};\Theta)$ which is the output of the $K$-phase network by taking $\zbf^{(i)}$ as the input. Here, the constraint parameter $\vartheta$ is set to be $10^{-3}$ in our experiment.
\section{Numerical Experiments}
\label{experiment}
To demonstrate the performance of the proposed algorithm and inspired network, we conduct extensive experiments on various image reconstruction problems and compare the results with some existing state-of-the-art algorithms.
Since the CNN in our design only provides a learnable regularization functional for the unrolled optimization algorithm, we adopt a step-by-step training strategy which imitates the iterating of optimization algorithm. More precisely, first we train the network with phase number $K = 3$, where each phase in the network corresponding to an iteration in optimization algorithm. After it converges, we add $2$ more phases to the end of it. Then with pretrained weights from $K=3$ we continue training the $5$-phase network until it converges, then $7$ phases, $9$ phases, etc., all the way until there is no noticeable improvement when we add more phases.

All the experiments in this section are performed on a machine with Nvidia GTX-1080Ti GPU of 11GB graphics card memory and implemented with the Tensorflow toolbox \cite{abadi2016tensorflow} in Python.
The learnable weights of convolutions are initialized by Xavier Initializer \cite{Glorot10understandingthe} and the threshold $\eta $ is initialized to be $0.01$. All the learnable parameters are trained by Adam Optimizer \cite{kingma2014adam}. The network is trained with learning rate 1e-4 for $500$ epochs when $K = 3$, followed by $200$ epochs when adding more phases.
Considering the graphics card memory and the cropped block size of images for training ($33 \times 33$ for nature image and $190 \times 190$ for MR image),  batch size $64$ and $2$ are decided when training the network with nature images and MR images respectively.
%
%
\subsection{Nature Image Compressive Sensing} \label{nature_image}
In this section, we conduct numerical experiments on nature image compressive sensing (CS) problems and compare the proposed ResGD-Net with some existing highly sophisticated methods.
For fair comparison, we use the same datasets among all methods, \textit{91 Images} for training and \textit{Set11} for testing \cite{kulkarni2016reconnet}.
The training sets are the extracted image luminance components which are then randomly cropped into $N = 88,912$ blocks of size $n = h \times w = 33^2$.
The experiments on different CS ratios $10 \%$, $25 \%$ and $50 \%$ are performed separately to compare the generality of the algorithms.
To create the data pairs $\{(\zbf^{(i)}, \xbf^{(i)})\}_{i=1}^{N}$ for training, where $\xbf^{(i)}$ is the image block and $\zbf^{(i)}$ is the CS measurement of $\xbf^{(i)}$, we first generate a random Gassuian measure matrix $\Phibf$ of size $ 10\%  n \times n, 25\%  n \times n $ and $ 50\%  n \times n$ whose rows are then orthogonalized, where this follows \cite{Zhang2018ISTANetIO};
then we apply $\zbf^{(i)} = \Phibf \xbf^{(i)}$ to generate the CS measurement.
When generating the testing data pairs from \textit{Set11} \cite{kulkarni2016reconnet}, we follow the same criterion as training data. All the testing results are evaluated on the average Peak Signal-to-Noise Ratio (PSNR) of the reconstruction quality.

\textbf{Comparison with some existing algorithms:}
In this part, we show the comparison results with some existing state-of-the-art algorithms, the variational methods TVAL3 \cite{li2013efficient}, D-AMP \cite{metzler2016denoising} and deep learning models IRCNN \cite{zhang2017learning}, ReconNet \cite{kulkarni2016reconnet} and ISTA-Net$^+$ \cite{Zhang2018ISTANetIO}. All the reconstruction results are tested on the avarage PSNR on \textit{Set11} \cite{kulkarni2016reconnet}, where the results are shown in Table \ref{cs-result}. Considering the trade-off between the network performance and complexity shown in the ablation study (Section \ref{para}), we determine the phase number $K = 19$ of our network when comparing with other algorithms. We observe that ResGD-Net outperforms all aforementioned algorithms by a large margin across all $10 \% $, $25 \%$ and $50 \%$ CS ratios.
In Fig. \ref{inception22-2} we show the reconstructed butterfly image with CS ratio $10\%$ and Barbara image with CS ratio $25\%$, it's clear that the proposed ResGD-Net is superior in preserving small patterns and details.
\begin{table}[t]
\vspace{-10pt}
\caption{Natural image CS reconstruction on data \textit{Set11} \cite{kulkarni2016reconnet} with CS ratios 10\%, 25\% and 50\%. Table shows the average PSNR (dB) of the comparison methods against ResGD-Net (19-phase). And the first five results of comparison algorithms are quoted from \cite{Zhang2018ISTANetIO}}
\label{cs-result}
\centering
\begin{tabular}{l|c|c|c}
\hline
\textbf{Algorithms} & \textbf{CS Ratio 10\%} & \textbf{CS Ratio 25\%} & \textbf{CS Ratio 50\%}\\ \hline
TVAL3 \cite{li2013efficient} & 22.99 & 27.92& 33.55 \\ \hline
D-AMP \cite{metzler2016denoising} & 22.64 & 28.46 &35.92 \\ \hline
IRCNN \cite{zhang2017learning} & 24.02 & 30.07 & 36.23\\ \hline
ReconNet \cite{kulkarni2016reconnet} & 24.28 & 25.60 & 31.50\\ \hline
ISTA-Net$^+$ (shared weights) \cite{Zhang2018ISTANetIO} & 26.51 & 32.08 & 37.59\\ \hline
ISTA-Net$^+$ \cite{Zhang2018ISTANetIO} & 26.64 & 32.57 & 38.07\\ \hline
\textbf{ResGD-Net} [Proposed] & \textbf{27.36} & \textbf{33.01} & \textbf{38.42}\\ \hline \end{tabular}
\end{table}
\begin{figure}[h]
\subfigure{
\label{fl-g} 
\includegraphics[width=0.152\linewidth]{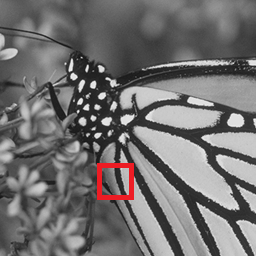}}
\subfigure{
\label{pa-g} 
\includegraphics[width=0.152\linewidth]{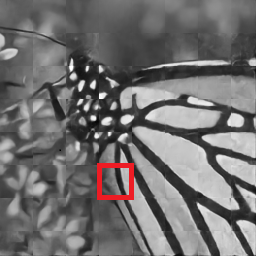}}
\subfigure{
\label{pa-epn} 
\includegraphics[width=0.152\linewidth]{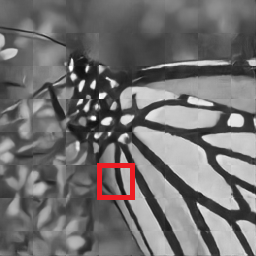}}
\subfigure{
\label{fl-g} 
\includegraphics[width=0.152\linewidth]{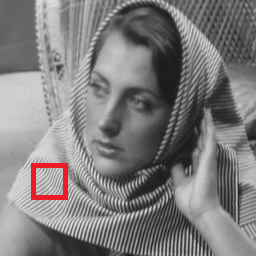}}
\subfigure{
\label{pa-g} 
\includegraphics[width=0.152\linewidth]{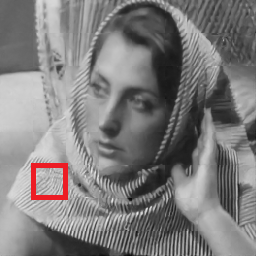}}
\subfigure{
\label{pa-epn} 
\includegraphics[width=0.152\linewidth]{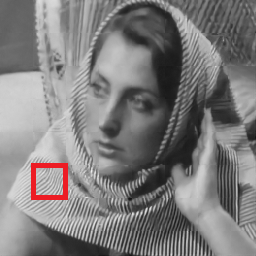}}
\setcounter{subfigure}{0}
\subfigure[\scriptsize{True}]{
\label{bu-ep-2} 
\includegraphics[width=0.15\linewidth]{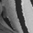}}
\subfigure[\scriptsize{ISTA-Net$^+$}]{
\label{bu-epn-2} 
\includegraphics[width=0.15\linewidth]{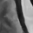}
}
\hspace{-5pt}
\subfigure[\scriptsize{ResGD-Net}]{
\label{bu-epn-2} 
\includegraphics[width=0.15\linewidth]{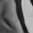}}
\subfigure[\scriptsize{True}]{
\label{bu-ep-2} 
\includegraphics[width=0.15\linewidth]{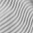}}
\subfigure[\scriptsize{ISTA-Net$^+$}]{
\label{bu-epn-2} 
\includegraphics[width=0.15\linewidth]{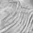}}
\subfigure[\scriptsize{ResGD-Net}]{
\label{bu-epn-2} 
\includegraphics[width=0.15\linewidth]{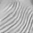}}
\vspace{-5pt}
\caption{Reconstruction results of a butterfly image with CS ratio 10\% and Barbara image with CS ratio 25\% in \textit{Set11} \cite{kulkarni2016reconnet}  using the state-of-the-art ISTA-Net$^+$ \cite{Zhang2018ISTANetIO} and the proposed ResGD-Net. PSNR and reconstruction time: (b) 25.91dB, 0.021s (c) 26.59dB, 0.237s (e) 29.21dB, 0.020s (f) 30.67dB, 0.225s}
\label{inception22-2} 
\vspace{-10pt}
\end{figure}
\subsection{Ablation study:} \label{ablation}
In this part, we chiefly do the ablation study to show the effectiveness of the residual connection, the influence of the number of phases over the results and the parameter efficiency of the proposed ResGD-Net. %

\textbf{The residual connection:}
To show the strength of the residual connection, we compare the test result of ResGD-Net against the gradient descent algorithm inspired network (GD-Net). The PSNR comparison is shown in Fig. \ref{phase-epoch} with various phase numbers $K$ and training epochs. We observe that with residual connection, ResGD-Net obtains much better quality of reconstructed images than the GD Net at each $K$.
As exemplified when $K$ fixed to be 3,
ResGD-Net converges with less training epoch number,
where ResGD-Net converges at around 250 epochs versus GD Net takes about 400 epochs.
%

\textbf{The phase number $K$:}
As shown in Fig. \ref{phase-epoch}, for both ResGD-Net and GD Net, PSNR increases with the increase of phase number $K$. The plot of ResGD-Net turns flat after 19 phases while GD Net does not tend to. Considering the trade-off between reconstruction performance and network complexity, we determine to take $K = 19$ when comparing ResGD-Net with other methods.

\textbf{The parameter efficiency:} \label{para}
The total number of parameters of GD-Net is $\{ A + B + \widetilde{A} + \widetilde{B} + \eta  + \alpha_k \times K = 32\times3\times3 \times(1+32 + 32) + 32\times3\times3 \times(1+32 + 32) + 1+ 19= 37,460\}$ if we take $K = 19$. Similarly, the total number of parameters of 19-phase ResGD-Net is $\{ A + B + \widetilde{A} + \widetilde{B} + \eta  + (\alpha_k +\gamma_k) \times K = 37,479\} $. The number of parameters per phase of ISTA-Net$^+$ is $37,442$ \cite{Zhang2018ISTANetIO}. It can be seen in Table \ref{cs-result} that ResGD-Net outperforms ISTA-Net$^+$ (shared weights) by a large margin (average 0.87 dB PSNR) with similar number of parameters. Even compared with ISTA-Net$^+$ with 9 phases unshared weights, ResGD-Net is still better (average 0.50 dB PSNR), whereas apparently there are far less parameters in ResGD-Net than unshared-weights ISTA-Net$^+$ (37,479 v.s. 336,978).
\begin{figure}[h]
\vspace{-4.5pt}
\centering
\subfigure{
\label{phase_num} 
\includegraphics[width=0.38\linewidth]{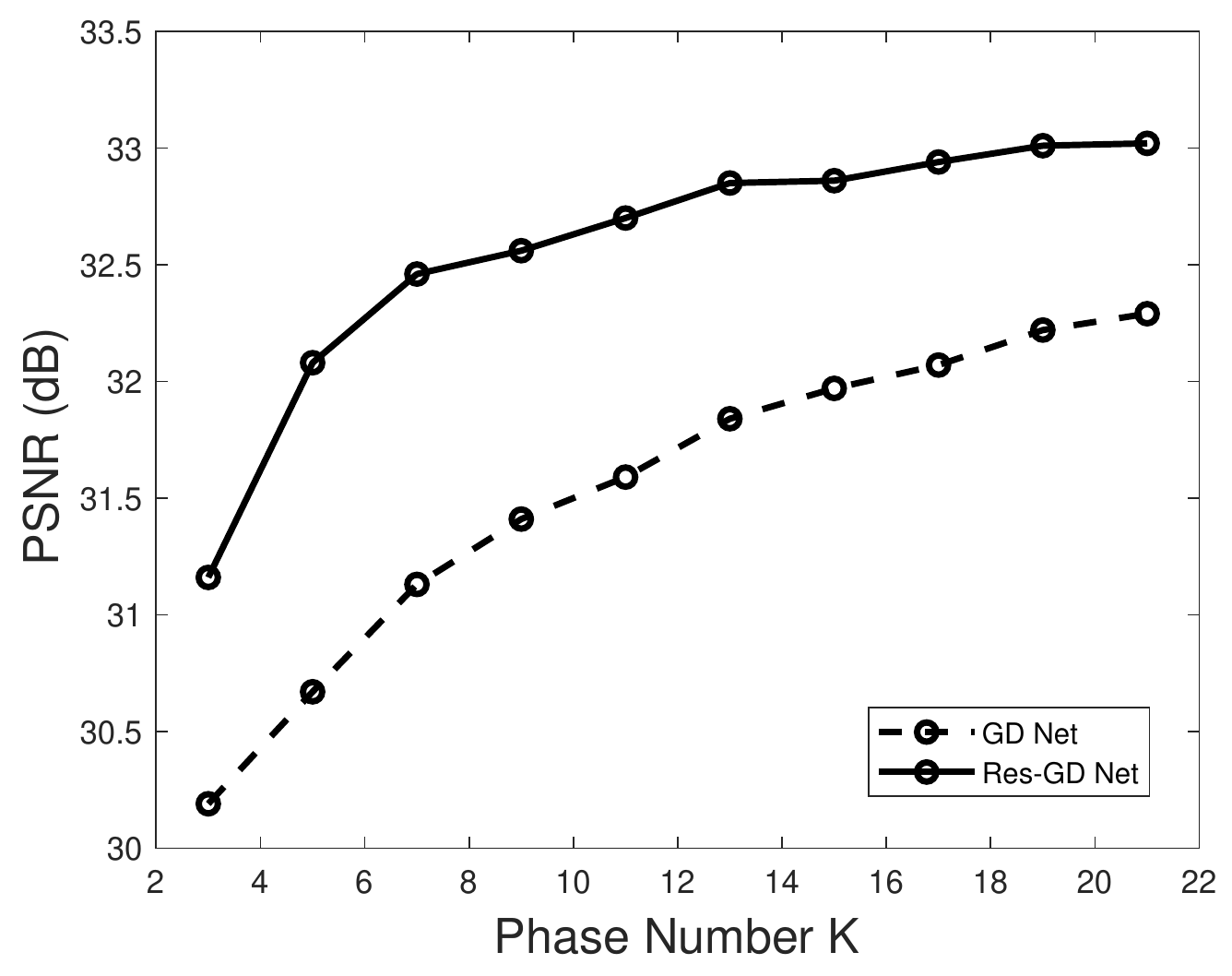}}
\hspace{20pt}
\subfigure{
\label{training_epoch} 
\includegraphics[width=0.39\linewidth]{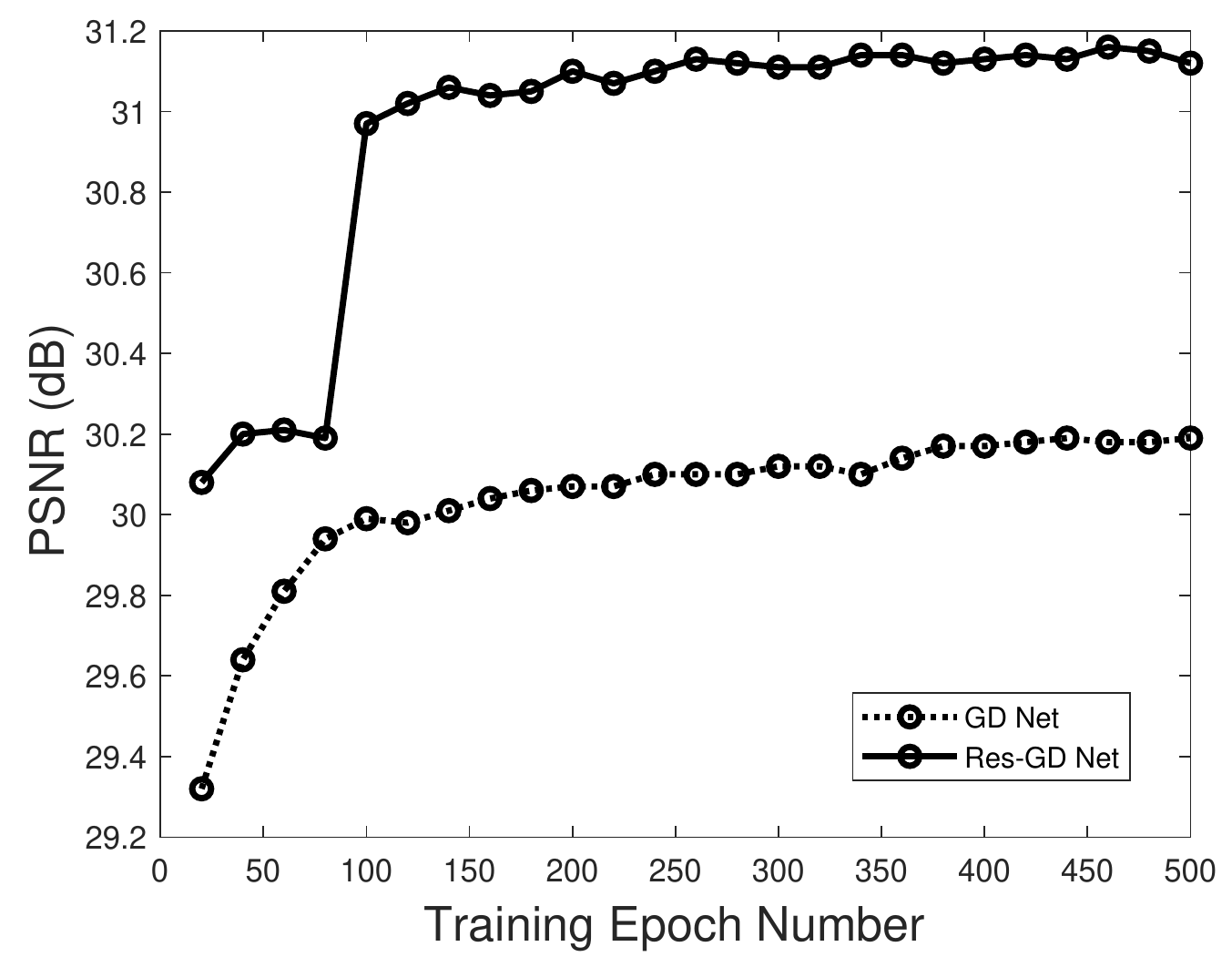}}
\vspace{-3pt}
   \caption{The PSNR comparison evaluated on \textit{Set11} \cite{kulkarni2016reconnet} between ResGD-Net and GD Net with various phase numbers and training epoch when CS ratio is $25\%$. Here, the evaluation on the training epoch is conducted on phase number $K = 3$}
\label{phase-epoch}
\end{figure}
\vspace{-10pt}
\subsection{Medical Image Compressive Sensing} \label{mri}
Medical image compressive sensing is an everlasting practical application in image reconstruction area.
In this section we test the performance of the proposed ResGD-Net on compressive sensing reconstruction of brain MR images \cite{mridata} (CS-MRI).
In CS-MRI problem, the data fidelity term is $f(\xbf;\zbf) = \|\Phi \xbf - \zbf \|^2_2$, where ${\Phi} = \mathcal{P}\mathcal{F}$, $\mathcal{P}$ is a binary selection matrix representating the sampling trajectory, and $\mathcal{F}$ is the discrete Fourier transform.
We randomly pick $150$ images from the brain MRI datasets \cite{mridata}, then crop and keep the central $190\times190$ part with less background. Then we at random divide the dataset to $100$ images for training and $50$ for testing. Among this section, we present the comparison results between ResGD-Net and ISTA-Net$^+$ \cite{Zhang2018ISTANetIO}, where the latter one is a state-of-the-art method in tackling with CS-MRI problem. For fairness, both algorithms compared here are evaluated on the same dataset and metrics. Experiments are conducted across different sampling ratios $10 \%$, $20 \%$ and $30 \%$ of $\mathcal{P}$ to show the generality.
The study of ResGD-Net on different sampling ratios and various phase numbers is shown in Fig. \ref{fig-mri}.
The PSNR comparison with ISTA-Net$^+$ is shown in Table. \ref{www}. The result enhancement of the proposed ResGD-Net against ISTA-Net$^+$ is remarkable across all sampling ratios even though we only use approximately $10 \%$ many number of parameters than ISTA-Net$^+$ \cite{Zhang2018ISTANetIO}.
%
\begin{figure}[h]
\vspace{-2pt}
\begin{center}
   \includegraphics[width=0.50\linewidth]{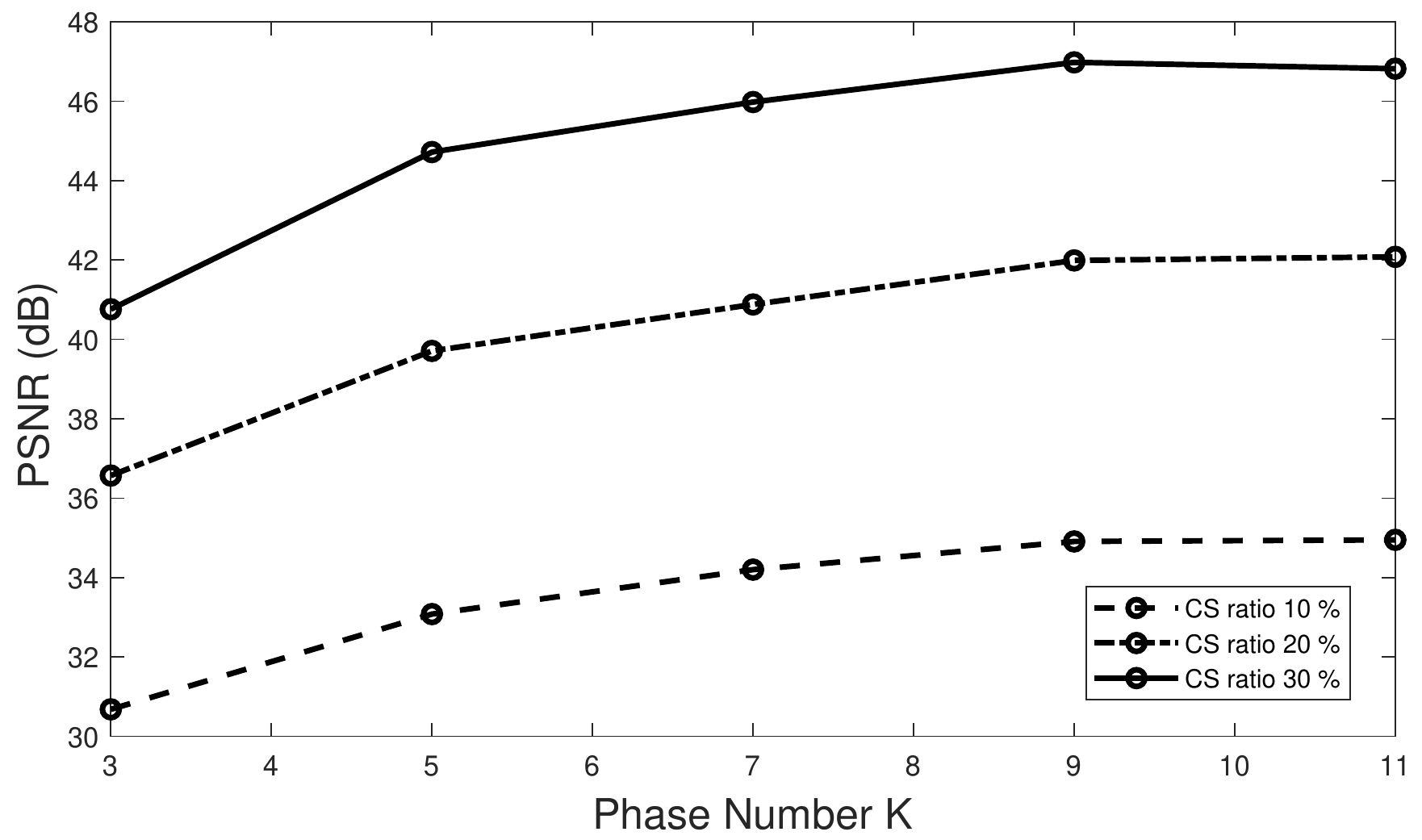}
\end{center}
\vspace{-4pt}
\caption{PSNR (dB) comparison of ResGD-Net on various phase numbers across different CS ratios $10 \%, 20 \%$ and $ 30 \%$ on brain MR images \cite{mridata}}
\label{fig-mri}
\end{figure}
%
\vspace{-20pt}
\begin{table}[h]
\vspace{-6pt}
\centering
\caption{PSNR (dB) of reconstructions obtained by ISTA-Net$^+$ \cite{Zhang2018ISTANetIO} and ResGD-Net (9 phases) on MR images using radial masks with different sampling ratios}
\label{www}
\begin{tabular}{c|c|c|c}
\hline
\textbf{Method}  & \textbf{Sampling ratio 10\%} & \textbf{Sampling ratio 20\%} & \textbf{Sampling ratio 30\%} \\ \hline
ISTA-Net$^+$ & 33.49 & 40.66 & 44.70 \\ \hline
ResGD-Net & 34.91 & 41.99 & 47.00 \\ \hline
\end{tabular}
\vspace{-6pt}
\end{table}

In addition, we provide the visualization results of some selected MR images reconstructed by the state-of-the-art ISTA-Net$^+$ \cite{Zhang2018ISTANetIO} and our proposed ResGD-Net on compressive sensing (CS) ratio $10 \%$, $20 \%$ and $30 \%$. The results are evaluated under metrics the Peak Signal-to-Noise Ratio (PSNR), the Structural Similarity (SSIM) and the Mean Squared Error (MSE). For better visualization, we rescale the pixel value by multiplying $8.0 \times$ on the error maps (the second row of Figs. \ref{inception11} - \ref{inception31}) when displaying.

\begin{figure}[h]
\subfigure{
\label{fl-g} 
\includegraphics[width=0.3125\linewidth]{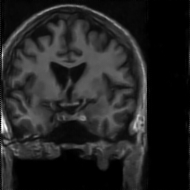}}
\subfigure{
\label{pa-g} 
\includegraphics[width=0.3125\linewidth]{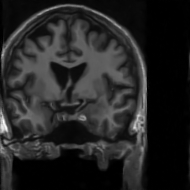}}
\subfigure{
\label{pa-ep} 
\includegraphics[width=0.3125\linewidth]{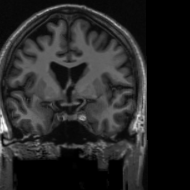}}
\setcounter{subfigure}{0}
\subfigure[{ISTA-Net$^+$}\protect \\PSNR: $29.09$dB \protect \\ SSIM: $0.8919$ \protect \\MSE: $1.231e-3$]{
\label{bu-epn-2} 
\includegraphics[width=0.3125\linewidth]{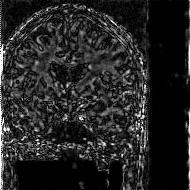}}
\hspace{0.5pt}
\subfigure[{ResGD-Net}\protect \\PSNR: $32.25$dB \protect \\ SSIM: $0.9178$ \protect \\MSE: $5.946e-4$]{
\label{bu-epn-2} 
\includegraphics[width=0.3125\linewidth]{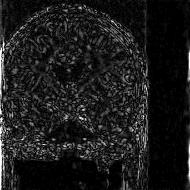}}
\subfigure[{True}]{
\label{bu-ep-2} 
\includegraphics[width=0.3125\linewidth]{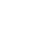}}
\caption{Reconstruction results of a brain MR image \cite{mridata} with radial mask of CS ratio 10\%  using the state-of-the-art ISTA-Net$^+$ \cite{Zhang2018ISTANetIO} and the proposed ResGD-Net. The figures in the second row are the difference images compared to the true image}
\label{inception11} 
\vspace{-10pt}
\end{figure}

\begin{figure}[h]
\subfigure{
\label{fl-g} 
\includegraphics[width=0.3125\linewidth]{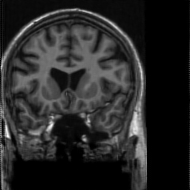}}
\subfigure{
\label{pa-g} 
\includegraphics[width=0.3125\linewidth]{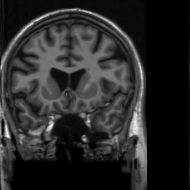}}
\subfigure{
\label{pa-ep} 
\includegraphics[width=0.3125\linewidth]{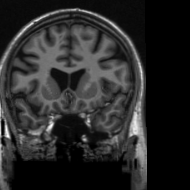}}
\setcounter{subfigure}{0}
\subfigure[{ISTA-Net$^+$}\protect \\PSNR: $31.51$dB \protect \\ SSIM: $0.9452$ \protect \\MSE: $7.069e-4$]{
\label{bu-epn-2} 
\includegraphics[width=0.3125\linewidth]{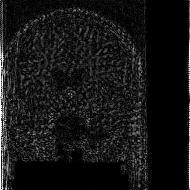}}
\hspace{0.5pt}
\subfigure[{ResGD-Net}\protect \\PSNR: $35.68$dB \protect \\ SSIM: $0.9595$ \protect \\MSE: $2.693e-4$]{
\label{bu-epn-2} 
\includegraphics[width=0.3125\linewidth]{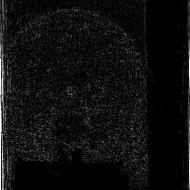}}
\subfigure[{True}]{
\label{bu-ep-2} 
\includegraphics[width=0.3125\linewidth]{fig/white.png}}
\caption{Reconstruction results of a brain MR image \cite{mridata} with radial mask of CS ratio 20\%  using the state-of-the-art ISTA-Net$^+$ \cite{Zhang2018ISTANetIO} and the proposed ResGD-Net. The figures in the second row are the difference images compared to the true image}
\label{inception22-2} 
\vspace{-10pt}
\end{figure}

\begin{figure}[h]
\subfigure{
\label{fl-g} 
\includegraphics[width=0.3125\linewidth]{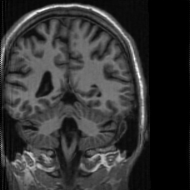}}
\subfigure{
\label{pa-g} 
\includegraphics[width=0.3125\linewidth]{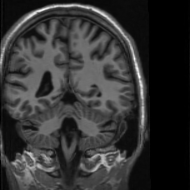}}
\subfigure{
\label{pa-ep} 
\includegraphics[width=0.3125\linewidth]{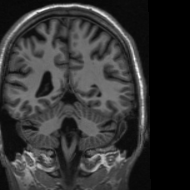}}
\setcounter{subfigure}{0}
\subfigure[{ISTA-Net$^+$}\protect \\PSNR: $35.73$dB \protect \\ SSIM: $0.9564$ \protect \\MSE: $2.671e-4$]{
\label{bu-epn-2} 
\includegraphics[width=0.3125\linewidth]{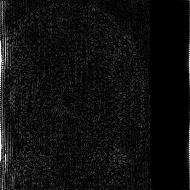}}
\hspace{0.5pt}
\subfigure[{ResGD-Net}\protect \\PSNR: $41.31$dB \protect \\ SSIM: $0.9772$ \protect \\MSE: $7.385e-5$]{
\label{bu-epn-2} 
\includegraphics[width=0.3125\linewidth]{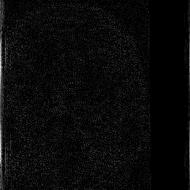}}
\subfigure[{True}]{
\label{bu-ep-2} 
\includegraphics[width=0.3125\linewidth]{fig/white.png}}
\caption{Reconstruction results of a brain MR image \cite{mridata} with radial mask of CS ratio 30\%  using the state-of-the-art ISTA-Net$^+$ \cite{Zhang2018ISTANetIO} and the proposed ResGD-Net. The figures in the second row are the difference images compared to the true image}
\label{inception31} 
\vspace{-10pt}
\end{figure}
\clearpage
 \section{Concluding Remarks}\label{sec:conclusion}
In this paper, motivated by Nestrov's smoothing technique and residual learning, we propose a residual learning inspired learnable gradient descent type algorithm with provable convergence. Then we present how to unroll the  algorithm into a deep neural network architecture. Furthermore, the proposed network is applied to different real-world image reconstruction applications. The numerical results show that our network outperforms several existing state-of-the-art methods by a large margin.

%
%

\bibliographystyle{splncs04}
\bibliography{egbib}
\end{document}